\documentclass{article} 

\usepackage[hyperindex=true,pageanchor=true,hyperfigures=true,backref=false]{hyperref}
\usepackage{amsmath,amssymb,theorem,url,dsfont} 
\usepackage{graphics,graphicx,epsfig}

\usepackage{geometry}
\usepackage{color}
\DeclareSymbolFontAlphabet{\Bbb}{AMSb}

\definecolor{lila}{rgb}{0.9,0,1}

\newcommand{\cR}[2]{{{\cal R}_{#1}{(#2)}}}

\newcommand{\cX}{\mathcal{X}}
\newcommand{\cY}{\mathcal{Y}}

\newcommand{\cXY}{\cX\times\cY}

\newcommand{\sx}[1]{}

\newcommand{\BlackBox}{\rule{1.5ex}{1.5ex}}

{\end{list}}

\newenvironment{proof}{\begin{list}{{\bf \em Proof: }}%
{\setlength{\labelsep}{0pt}\setlength{\leftmargin}{0pt}\setlength{\labelwidth}{0pt}}\item}%
{\hfill\BlackBox\end{list}}

{\end{list}}

{\hfill\BlackBox\end{list}}
{\hfill\BlackBox\end{list}}
{\hfill\BlackBox\end{list}}

\newcommand{\bc}{\begin{center}}
\newcommand{\ec}{\end{center}}
\newcommand{\bi}{\begin{itemize}}
\newcommand{\ei}{\end{itemize}}
\newcommand{\be}{\begin{equation}}
\newcommand{\ee}{\end{equation}}
\newcommand{\beqna}{\begin{eqnarray*}}
\newcommand{\eeqna}{\end{eqnarray*}}
\newcommand{\bd}{\begin{displaymath}}
\newcommand{\ed}{\end{displaymath}}
\newcommand{\bt}{\begin{tabular}}
\newcommand{\et}{\end{tabular}}

\newtheorem{Definition}{Definition}[section]
\newtheorem{Theorem}[Definition]{Theorem}

\newtheorem{Example}[Definition]{Example}

\newtheorem{Lemma}[Definition]{Lemma}

\newtheorem{Assumption}[Definition]{Assumption}

\newcommand{\rem}[1]{}
\newlength{\fixboxwidth}
\newcommand{\N}{\mathds{N}}    

\newcommand{\R}{\mathds{R}}    

\newcommand{\snorm}[1] {\Vert #1 \Vert}

\newcommand{\inorm}[1] {\Vert #1 \Vert_\infty}

\newcommand{\lt}{<}

\def \lb        { \lambda }

\def \s         { \sigma }

\newcommand{\LnullX} {{\mathcal{L}_0(\cX)}}  

\def \P       {\mathrm{P}}   
\def \Q         { \mathrm{Q} } 

\def \D         { \mathrm{D} }
\newcommand{\PM} {\mathcal{M}_1}   

\newcommand{\Ex}{\mathbb{E}}     

\newcommand{\lpin}{{{L}_{\tau\mathrm{-pin}}}}

\newcommand{\Lpin}[1]{{L_{#1-\mathrm{pin}}}}

\newcommand{\cL}{\mathcal{L}}

\newcommand{\LZrisk}{{\cal R}_{L_0,{\P}}}
\newcommand{\LErisk}{{\cal R}_{L_\epsilon,{\P}}}

\newcommand{\fPn}{f_{\P;n}}
\newcommand{\fPOn}{f_{\P_1;n}}
\newcommand{\fDnn}{f_{\D;n}}
\newcommand{\gPPn}{g_{\P,\P;n}}
\newcommand{\gPOPTn}{g_{\P_1,\P_2;n}}
\newcommand{\gPPOn}{g_{\P,\P_1;n}}
\newcommand{\gPDn}{g_{\P,\D;n}}

\newcommand{\gDDn}{g_{\D,\D;n}}

\title{\textbf{Estimation of scale functions to model heteroscedasticity by support vector machines}}

\author{\textbf{Robert Hable}\\
University of Bayreuth\\
Department of Mathematics\\
D-95440 Bayreuth\\
\texttt{robert.hable@uni-bayreuth.de}
\\ \hspace*{1em} \\
\textbf{Andreas Christmann}\\
University of Bayreuth\\
Department of Mathematics\\
D-95440 Bayreuth\\
\texttt{andreas.christmann@uni-bayreuth.de}
}

\begin{document}

\maketitle

\begin{abstract}

A main goal of regression is to derive statistical conclusions on the conditional distribution of the output variable $Y$ given the input values $x$. Two of the most important characteristics of a single distribution are location and scale. Support vector machines (SVMs) are well established to estimate location functions like the conditional median or the conditional mean. We investigate the estimation
of scale functions by SVMs when the conditional median is unknown, too. Estimation of scale functions is important e.g. to estimate the volatility in finance.
We consider the median absolute deviation (MAD) and the interquantile range (IQR)
as measures of scale. 
Our main result shows the consistency of MAD-type SVMs.
\end{abstract}

\section{Introduction}\label{sec:intro}
  
Let $\P$ be the distribution of a pair of random variables
$(X,Y)$ with values in a set
$\cX\times\cY$ where $X$ is an input variable and $Y$ is
a real-valued output variable. 
The goal in regression problems is to derive statistical conclusions on the conditional distribution of $Y$ given $X=x$. 
Generally, location and scale are considered as the two most important characteristics of a distribution and estimating these quantities is one of the main topics in statistics. 

Regularized empirical risk minimization
\cite{vapnik1998,wahba1999,ScSm2002,Hofmann:Schoelkopf:2008}
using the kernel trick proposed by \cite{ScSmMu98a}
and the special case of \textit{support vector machines} (SVMs)
\cite{vapnik1998,CristianiniShaweTaylor2000,ScSm2002,SteinwartChristmann2008a} 
are well established methods in order to estimate the location
of the conditional distribution of $Y$ given $X=x$. For an
i.i.d.\ sample
$D=\big((X_1,Y_1),\dots,(X_n,Y_n)\big)$ 
drawn from $\P$, the SVM-estimator is defined by
\begin{eqnarray}\label{intro-def-svm}
  f_{L,\D,\lambda}\;=\;
  \text{arg}\inf_{f\in H}
    \frac{1}{n}\sum_{i=1}^n L\big(Y_i,f(X_i)\big)
    +\lambda\|f\|_H^2 \, ,
\end{eqnarray}
where $L$ is a loss function, $H$ is a certain space
-- a so-called \textit{reproducing kernel Hilbert space}
(RKHS) -- 
of functions $f:\cX\rightarrow\R$, and $\lambda\in(0,\infty)$ is
a regularization parameter in order to prevent overfitting.
We refer to  
\cite{wahba1999,ScSm2002,BerlinetThomasAgnan2004,Hofmann:Schoelkopf:2008,SteinwartChristmann2008a}
for the concept of an RKHS.
There are a number of different quantities which describe
the location of a single distribution and which can be estimated by
SVMs by choosing a suitable loss function. The conditional
mean function $g(x):=\Ex_\P[Y|X=x]$, $x\in\cX$ can be estimated by
using the least-squares loss $L_{LS}(y,t)=(y-t)^2$ and the 
\emph{$\tau$-quantile function} $g(x):=f_{\tau,\P}^\ast(x)$, $x\in\cX$,
(see (\ref{def-quantile-function}) below)
by using the \emph{$\tau$-pinball loss function}
$$
\lpin(y,t)\;=\;
  \left\{\renewcommand{\arraystretch}{0.9}
    \begin{array}{ll}
      (1-\tau)\cdot(t-y) & \,\text{ if }\;y-t < 0, \\
      \tau\cdot(y-t) & \,\text{ if }\;y-t \geq 0, \qquad (y,t)\in\cY\times\R, \\
    \end{array}
\right.
$$
see \cite{KoenkerBassett1978,Koenker2005,ScSmWiBa00aa,TakeuchiLeSearsSmola2006}.
The choice $\tau=0.5$ leads to an estimate of 
the \emph{median function}
$$
f_{0.5,\P}^\ast(x):=\text{median}_\P(Y|X=x), \qquad x\in\cX.
$$

The goal of this paper is to investigate two methods to estimate the variability of the conditional distributions of $Y$ given $X=x$ for $x\in\cX$ via \emph{scale functions}.
Estimation of heteroscedasticity is interesting in many areas of applied statistics, e.g., for the estimation of volatility in finance.  
To fix ideas, let us \emph{illustrate} what we mean by scale function estimation by considering a small data concerning the so-called LIDAR technique.
LIDAR is the abbreviation of LIght Detection And Ranging. 
This technique uses the reflection of laser-emitted light to detect chemical compounds in the atmosphere. We consider the logarithm of the ratio of light received from two laser sources as the output variable $Y=\mathtt{logratio}$, whereas the single input variable $X=\mathtt{range}$ is the distance traveled before the light is reflected back to its source.  We refer to \cite{RuppertWandCarroll2003}  
for more details on this data set.
A scatterplot of the data set consisting of $n=221$ observations is shown in the left subplot of Figure \ref{lidarplot} together with the fitted quantile curves based on SVMs using the pinball loss function for $\tau\in\{0.05, 0.25, 0.5, 0.75, 0.95\}$ and the Gaussian RBF kernel $k(x,x'):=\exp(-\gamma \snorm{x-x'}_2^2)$ for $x,x'\in\cX$.
By looking at the estimated median function (i.e., the black curve in the middle of
the left subplot), we clearly see a nonlinear relationship between both variables which is almost constant for values of $\mathtt{range}$ below $550$ and decreasing for higher values of \texttt{range}. However, there is also a considerable change of the variability of $\mathtt{logratio}$ given $\mathtt{range}$: the variability
is relatively small for small values of $\mathtt{range}$, but much larger for large values of $\mathtt{range}$. This becomes obvious by looking at the other estimated quantile curves in the left subplot or by looking at the right subplot of Figure
\ref{lidarplot} which shows the estimated width of intervals covering at least $50\%$ of the mass of $\P(Y|x)$. In this simple example we can just look onto the 2-dimensional plot to realize this kind of heteroscedasticity of the conditional distribution of $Y$ given $X=x$. However, this is obviously no longer possible if the input space $\cX$ is a high-dimensional Euclidean space or
an abstract metric space. Hence an automatic and non-parametric method to model and to estimate such kind of variability becomes important.   
\begin{figure}[t]\label{lidarplot}
\begin{center}
\caption{Illustration of the estimation for scale functions by SVMS for the LIDAR data set. Left subplot: data set, estimated quantile functions with SVMs for $\tau=0.5$ (black), $\tau=0.25$ and $\tau=0.75$ (both in blue),
$\tau=0.05$ and $\tau=0.95$ (both in red).
Right subplot: Estimated width of the intervals covering $50\%$ of the mass of $\P(Y|x)$. IQR-type SVM (blue) using $(\tau_1,\tau_2)=(0.25,0.75)$ and
$2$ times the MAD-type SVM (green).}
\vspace{-1cm}
 \includegraphics[width=0.80\textwidth]{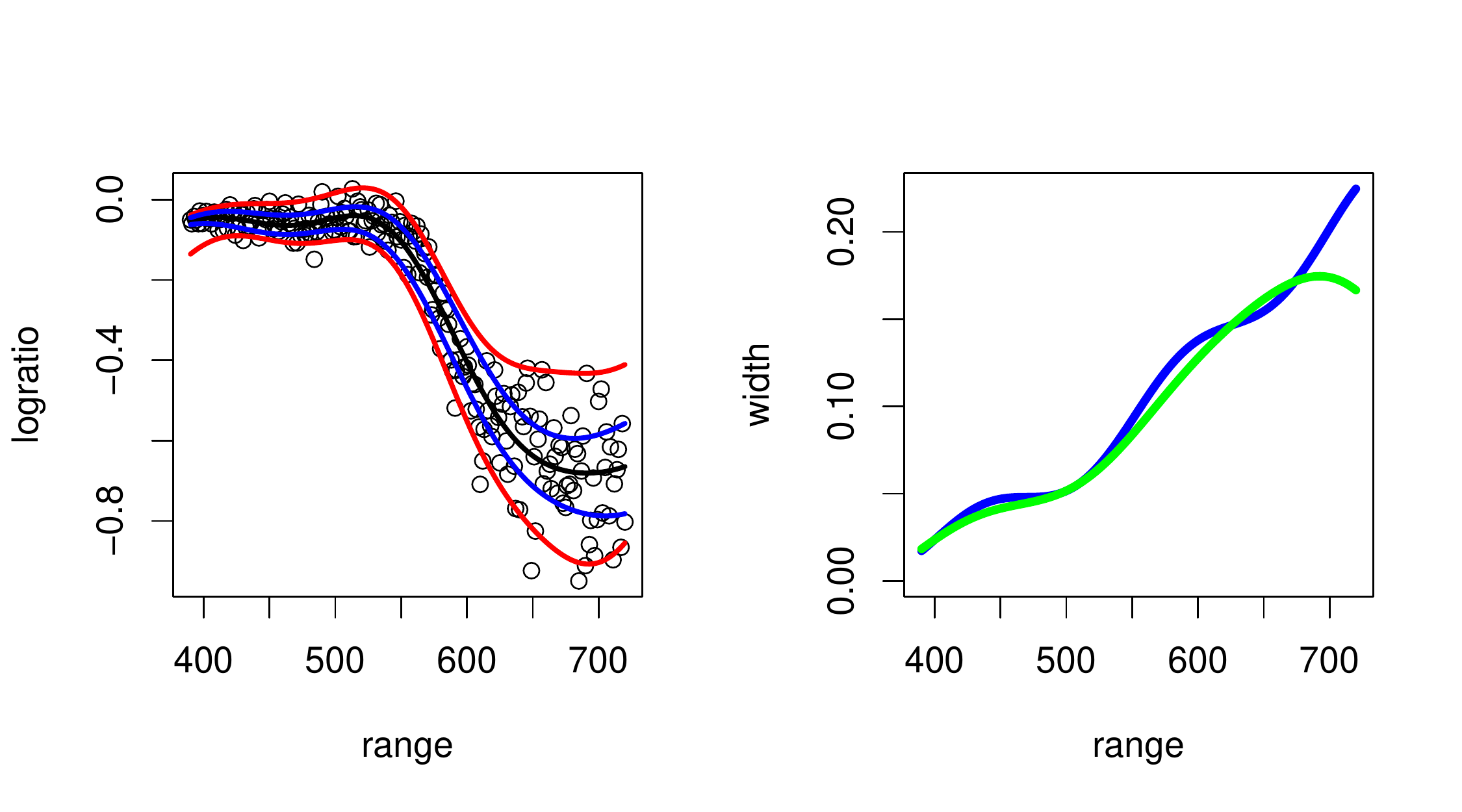}
\end{center} \vspace{-1cm}
\end{figure}
Therefore, this article investigates how two classical scale quantities of the conditional distribution of $Y$ given $X=x$ can be estimated by use of SVMs. 
Such scale functions $g:\cX\to[0,\infty)$ are quite common in a heteroscedastic model like $\P(Y|x)=f(x)+g(x)\varepsilon$, where $f$ denotes the location function and $\varepsilon$ denotes the stochastic error term. Note, that we will \emph{not} assume 
such a specific model.
As in case of location, there are
several well established quantities which describe the scale, e.g., \cite[Chap.\,5]{Huber1981}\\
$~~~(i)~~\,$ \emph{the variance function}:
$g(x)=\text{Var}_\P(Y|X=x)$, $x\in\cX$,\\
$~~~(ii)~$ \emph{the median absolute
deviation from the median (MAD) function}: \\ 
$~~~~~~~~~~~ g(x):=\text{MAD}_\P(Y|X=x)
 :=\text{median}\big(|Y-f_{0.5,\P}^\ast(x)|\,\big|\,X=x\big),$ $x\in\cX$,\\
$~~~(iii)$ \emph{the interquantile range (IQR) function for quantiles $\tau_1 <\tau_2$}:\\
$~~~~~~~~~~~ g(x):=\text{IQR}_{\tau_1,\tau_2}(Y|X=x):=
 f_{\tau_2,\P}^\ast(x)-f_{\tau_1,\P}^\ast(x)$, $x\in\cX$.\newline
Note that these three quantities are \emph{not} directly comparable. However, 
$\mathrm{IQR}_{0.25,0.75}$ and $2$ times MAD
are both quantities for the width of an interval covering at least $50\%$ of the probability mass of
$\P(Y|x)$.
There is a vast literature on the estimation of scale functions, often based on special parametric dispersion models, see, e.g.,
\cite{Jorgensen1987,Smyth1989,Jorgensen1997},
and for a wavelet thresholding approach for univariate 
regression models we refer to \cite{Cai:Wang:2008}. 

In this article, we consider the MAD function and the  IQR function and show how both can be consistently estimated in a purely nonparametric manner with SVMs.  
In case of the MAD,
we estimate the unknown median function $f_{0.5,\P}^\ast$
by an SVM $f_{\Lpin{0.5},\D,\lambda}$
and calculate the estimated absolute residuals
$
R_i:=|Y_i-f_{\Lpin{0.5},\D,\lambda}(X_i)|
$
in a first step. 
In a second step, we estimate the conditional median of the absolute residuals by the SVM
based on a smoothed version of the $\frac{1}{2}$-pinball loss defined in (\ref{def-smoothed-pinball-loss-function}) below
for the pairs of random variables $(X_i,R_i)$. The resulting estimator is called 
MAD-type SVM and it is shown
in Subsection \ref{sec:results2} that it is risk-consistent (up to any predefined
$\varepsilon>0$) 
even though 
(i) the estimation in the second step
cannot be based on the true residuals but has to be based on the
estimated residuals because the true median function is unknown
and (ii) the random variables $(X_i,R_i)$ are not i.i.d. 
In case of the IQR 
$f_{\tau_2,\P}^\ast-f_{\tau_1,\P}^\ast$, we respectively estimate the 
 $\tau_j$-quantile function $f_{\tau_j,\P}^\ast$
by use of the $\tau_j$-pinball loss so that we get 
$f_{L_{\tau_2\text{-pin}},\D,\lambda_2}-
 f_{L_{\tau_1\text{-pin}},\D,\lambda_1}
$
as an estimate, which we call IQR-type SVM. As this is the difference of two
standard SVMs, we can carry over many well-known facts on SVMs 
in this case in Subsection \ref{sec:results1}.
In both cases, available software, e.g., the R-package \texttt{kernlab} \cite{kernlab}
or the C++ implementation \texttt{mySVM}
\cite{Rueping2000}, can be used since we essentially have to calculate SVMs for pinball losses. 

The rest of the paper has the following structure.
Section 2 contains with Theorem \ref{MAD-risk-consistency}
our main result.
Section 3 contains not only the proof of this theorem, but also gives
two new consistency results in the $L_1$-sense for SVMs based on the pinball loss, see Lemma \ref{lemma-L1-convergence-quantile}
and Theorem \ref{rh-L1consistency}. Although we need these results in our proof of
Theorem \ref{MAD-risk-consistency}, we think that they 
are interesting in its own, because
they improve earlier consistency results of SVMs which showed the weaker kind of convergence in probability, see \cite[Cor.\ 3.62, Thm.\ 9.7(i)]{SteinwartChristmann2008a}.

\section{Main results}\label{sec:results}
    
The following assumptions and notations are 
used throughout the whole article.
\begin{Assumption}\label{genass}
Let $\cX$ be a complete separable metric space, e.g.\
$\cX=\R^d$, and
$\cY\subset\R$ be closed. For $j\in\{1,2\}$,
let $k_j:\cX\times\cX\to\R$ be a bounded 
continuous kernel with
$\inorm{k_j}:=\sup_{x\in\cX}(k_j(x,x))^{1/2}<\infty$.
Its corresponding reproducing kernel Hilbert space (RKHS)
is denoted by $H_j$, its corresponding canonical feature map
is denoted by $\Phi_j$, and it is assumed that
each $H_j$ is dense in $L_1(\mu)$ for every $\mu\in\PM(\cXY)$. 
\end{Assumption}
$\PM(\cXY)$ denotes the set of all Borel probability measures on $\cXY$.
The unknown joint distribution of $(X,Y)$ is denoted by $\P\in\PM(\cXY)$
and $D=\big((X_1,Y_1),\dots,(X_n,Y_n)\big)$ is an i.i.d.\ sample
drawn from $\P$.
Let $\P_{\cX}$ denote the distribution of $X$, let
$\LnullX$ denote the set of all Borel measurable functions $f:\cX\to\R$ and
let $L_1(\P_{\cX})$ denote the set of all $\P_{\cX}$-integrable
functions $f:\cX\to\R$.
We define the $\tau$-quantile function as the (perhaps set-valued)
function
\begin{eqnarray}\label{def-quantile-function}
  \cX\,\rightarrow\,2^\R\,,\;\;
  x\,\mapsto\,F_{\tau,\P}^\ast(x)\,:=\,
  \big\{t\in\R\,:\,\,\P\big((-\infty,t]\big|\,x\big)\geq\tau
        \text{ and }\P\big([t,\infty)\big|\,x\big)\geq1-\tau
  \big\}\,.
\end{eqnarray}
We make the standard assumption that $F_{\tau,\P}^\ast(x)$ are 
singletons and hence we write 
$f_{\tau,\P}^\ast(x):\cX\rightarrow\R$ instead,
see \cite{SteinwartChristmann2007b,SteinwartChristmann2011}.

  \subsection{MAD-type estimation}\label{sec:results2}

We would like to estimate the MAD function given by
$g(x)=\text{MAD}_{\P}(Y|X=x)
 =\text{median}_{\P}\big(|Y-f_{0.5,\P}^\ast(x)|\,\big|\,X=x\big)
$
where $f_{0.5,\P}^\ast$ is the 
median function. First, we estimate the median function $f_{0.5,\P}^\ast$.
For a random sample 
$D=\big((X_1,Y_1),\dots,(X_n,Y_n)\big)$ drawn from $\P$, 
the SVM-estimator
for $f_{0.5,\P}^\ast$
is 
$$
  f_{\Lpin{0.5},\D,\lambda_{1,n}}\;=\;
  \text{arg}\inf_{f\in H_1} 
  \bigg(\frac{1}{n}\sum_{i=1}^n 
            \Lpin{0.5}\big(Y_i,f(X_i)\big)
        \,+\,\lambda_{1,n}\|f\|_{H_1}^2
  \bigg),
$$
$\lambda_{1,n}\in(0,\infty)$, and $H_1$ is an RKHS.
Then, we can estimate the conditional median of the absolute residuals
$|Y-f_{0.5,\P}^\ast(x)|$ by use of the \emph{estimated} 
absolute residuals.
Let us define
\begin{eqnarray}\label{def-MAD-SVM-pre}
  \tilde{g}_{\D,n}\;=\;
  \text{arg}\inf_{g\in H_2} 
  \bigg(\frac{1}{n}\sum_{i=1}^n 
                     L_{\varepsilon}
                        \big(\big|Y_i-f_{\Lpin{0.5},\D,\lambda_{1,n}}(X_i)
                             \big|,
                             g(X_i)
                         \big)
        \,+\,\lambda_{2,n}\|g\|_{H_2}^2
  \bigg),
\end{eqnarray}
where, for some small predefined number $\varepsilon>0$, 
the loss function $L_{\varepsilon}$ defined by 
\begin{eqnarray}\label{def-smoothed-pinball-loss-function}
  L_{\varepsilon}(y,t)\;=\;
  \tfrac{1}{2}(y-t)-
  \varepsilon\log\big(2\Lambda\big(\tfrac{y-t}{\varepsilon}\big)
                 \big)
  \;=\;
    \Lpin{0.5}(y,t)-
  \varepsilon
     \log\big(2\Lambda\big(\tfrac{|y-t|}{\varepsilon}\big)
         \big)\;,
\end{eqnarray}
is an $\varepsilon$-smoothed version
of the pinball loss function for $\tau=0.5$,
$\Lambda(r)=1/(1+e^{-r})$ for every $r\in\R$,
$\lambda_{2,n}\in(0,\infty)$,
and $H_2$ is an RKHS. Since $\tilde{g}_{\D,n}$ occasionally can
have negative values, we propose the MAD-type estimator
\begin{eqnarray}\label{def-MAD-SVM-non-negative}
  g_{\D,n}\;=\;\max\{\tilde{g}_{\D,n},0\}
\end{eqnarray}
instead of $\tilde{g}_{\D,n}$.
We use the smoothed version $L_\varepsilon$ of the pinball loss function
$\Lpin{0.5}$
because we will need in the proof of Theorem
\ref{MAD-risk-consistency} that the loss function
has a Lipschitz continuous derivative,
see (\ref{lipschitz-continuous-derivative-of-loss-1}) and 
(\ref{lipschitz-continuous-derivative-of-loss-2}).
This is the price we pay for the
unavoidable fact that
our estimation cannot be based on the true residuals but on the 
estimated ones because the distribution $\P$ of
$(X_i,Y_i)$ is assumed to be unknown in statistical machine learning.
Some easy calculations show that the smoothed pinball loss function $L_\varepsilon$ 
is convex, Lipschitz continuous with $|L_\varepsilon|_1=0.5$, has a Lipschitz continuous derivative, and fulfills 
$0\leq \Lpin{0.5}(y,t)-L_{\varepsilon}(y,t)
 \leq\log(2) \varepsilon<\varepsilon
$
for every $(y,t)\in\cY\times\R$ 
and the risks fulfill, for all $\P\in\PM(\cY\times\R)$,
$$0\;\leq\;
 \Ex_\P \Lpin{0.5}(Y,f(X))- \Ex_\P L_\varepsilon(Y,f(X))
 \;\le\;
 \Ex_\P |\Lpin{0.5}(Y,f(X))-L_{\varepsilon}(Y,f(X))|
 <
 \varepsilon\;.
$$
The $\varepsilon$-smoothed version of the pinball loss is actually
a re-parametrized logistic loss function 
$L_\varepsilon(y,t)
 =\varepsilon L_{\text{logistic}}(y/\varepsilon,t/\varepsilon)/2,
$  
see \cite[p.\ 44]{SteinwartChristmann2008a}
Hence the SVM based on $L_\varepsilon$ can be calculated
by any software which supports the logistic loss. For the
illustration purposes in the introduction, we used $\varepsilon=0.1$
and calculated (\ref{def-MAD-SVM-pre}) by Newton-Raphson.

For any loss function $L$ and
every measurable $f,g:\cX\rightarrow\R$, define the risk
\begin{eqnarray}\label{def-risk-for-mad}
  \mathcal{R}_{L,\P}(f,g)\;:=\;
  \Ex_{\P}\, L\big(|Y-f(X)|,g(X)\big)\;.
\end{eqnarray}
If the median function 
$f_{0.5,\P}^\ast$ and the 
MAD function 
$g_{\P}^\ast(x)=\text{MAD}(Y|X=x)$ 
uniquely exist, then 
the 
MAD function $g_{\P}^\ast$ minimizes
$g\mapsto\mathcal{R}_{\Lpin{0.5},\P}\big(f_{0.5,\P}^\ast,g\big)
$
over all measurable functions $g:\cX\rightarrow\R$, i.e.\
$$\mathcal{R}_{\Lpin{0.5},\P}\big(f_{0.5,\P}^\ast,g_{\P}^\ast\big)\;=\;
  \inf_{g\in\cL_0(\cX)}
    \mathcal{R}_{\Lpin{0.5},\P}\big(f_{0.5,\P}^\ast,g\big)
  \;.
$$

The following theorem says that $g_{\D,n}$ is
risk $\varepsilon$-consistent for the MAD function.
\begin{Theorem}\label{MAD-risk-consistency}
  In addition to Assumption 
  \ref{genass}, assume that $\Ex_{\P} |Y|<\infty$
  and that the median function
  $f_{0.5,\P}^\ast:\cX\rightarrow\R$ is almost surely unique.
  Let $L_0$ denote the $0.5$-pinball loss function and
  let $\varepsilon>0$ be the predefined real number in the
  loss function $L_\varepsilon$.
  Then, for $n\rightarrow\infty$,
  $$\inf_{g\in\cL_0(\cX)}\!\!
    \LZrisk\big(f_{0.5,\P}^\ast,g\big)
    +\varepsilon\,\geq\,
    \LZrisk\big(f_{0.5,\P}^\ast,g_{\D,n}\big)+o_{\P}(1)
    =\LZrisk\big(f_{L_0,\D,\lambda_{1,n}},g_{\D,n}\big)+o_{\P}(1)
  $$
  if $\,\lim_{n\rightarrow\infty}\lambda_{1,n}=0$,
  $\,\lim_{n\rightarrow\infty}\lambda_{2,n}=0$, \,and
  $\,\lim_{n\rightarrow\infty}\lambda_{1,n}^2\lambda_{2,n}^2n=\infty$.
\end{Theorem}
\textbf{Remarks:} 
(i) We assume that the true
median function uniquely exists but we do \emph{not} assume 
that the true
MAD function $g_{\P}^\ast$ uniquely exists. 
(ii) The value
$\LZrisk\big(f_{0.5,\P}^\ast,g_{\D,n}\big)$ quantifies the 
expected
distance of the estimate $g_{\D,n}$ to the absolute values of the
\emph{true} residuals $|Y-f_{0.5,\P}^\ast(X)|$;
the value
$\LZrisk\big(f_{L_0,\D,\lambda_{1,n}},g_{\D,n}\big)$ quantifies the 
expected
distance of the estimate $g_{\D,n}$ to the absolute values of the
\emph{estimated} residuals $|Y-f_{L_0,\D,\lambda_{1,n}}(X)|$. 
According to Theorem \ref{MAD-risk-consistency}, both values
asymptotically achieve the infimal risk up to the predefined 
$\varepsilon>0$.
(iii) The assumption
$\lim_{n\rightarrow\infty}\lambda_{1,n}^2\lambda_{2,n}^2n=\infty$ 
is stronger than the standard assumption
$\lim_{n\rightarrow\infty}\lambda_{j,n}^2=\infty$; see
\cite[Thm.\ 9.6]{SteinwartChristmann2008a}. This is plausible because
estimating the MAD is burdened with the estimation
of a nuisance function
(i.e. the unknown median function).

  \subsection{IQR-type estimation}\label{sec:results1}
    
Let us now consider a linear combination of $m$ SVMs under the Assumption \ref{genass}.
As the results follow by straightforward calculations using standard results on SVMs, the proofs are left out.

Let $m$ be a positive integer, $J=\{1,\ldots,m\}$, 
$c=(c_1,\ldots,c_m)\in\R^m\!\setminus\!\{0\}$, and
$(\xi_{j,n})_{n\in\N_0}$ be a sequence of measurable functions  into some complete separable metric space $E_j$ enclipped with its Borel $\s$-algebra, $j\in J$.
Obviously, $\sum_{j\in J} c_j \xi_{j,n}$ exists and is unique if
all $\xi_{j,n}$ exist and are unique. 
Furthermore, $\sum_{j\in J} c_j \xi_{j,n}$ converges to 
$\sum_{j\in J} c_j \xi_{j,0}$ in probability (or almost surely or in the
$L_p$ sense) if all $\xi_{j,n}$ converge in probability (or almost surely or in the 
$L_p$ sense) to $\xi_{j,0}$ for $n\to \infty$. 

Now, let $0 \lt \tau_1 \lt \ldots \lt \tau_m \lt 1$. If we either specialize that $\xi_{j,n}$ denotes the support vector machine $f_{\Lpin{\tau_j}, \D, \lb_{j,n}}$ and choose as $E_j$ the RKHS $H_j$ or that $\xi_{j,n}$ denotes the corresponding
risk $\Ex_\P \Lpin{\tau_j}(Y,f_{\Lpin{\tau_j}, \D, \lb_{j,n}}(X))$ and choose $E_j=\cY$, then 
existence, uniqueness and consistency results for the linear combination of the
SVMs or of their risks follow immediately from results 
valid for each individual SVM, see, e.g., \cite{SteinwartChristmann2008a,SteinwartChristmann2007b,SteinwartChristmann2011}
and our Theorem \ref{rh-L1consistency}.
Denote the sub\-differential (see e.g. \cite[Section 5.3]{denkowski2003}) of the pinball loss function $\Lpin{\tau_j}$ 
(with respect to the second argument) by
$\partial \Lpin{\tau_j}$.
We then obtain immediately a representer theorem for the linear combination of SVMs because it is well-known that each individual SVM has a representer theorem, i.e.
it holds 
\be
\sum_{j\in J} c_j f_{\Lpin{\tau_j}, \P, \lb_{j,n}} 
= 
\sum_{j\in J}
-\frac{1}{2\lb_j} c_j \Ex_{\P} h_{j,\P}(X,Y)\Phi_j(X),
\ee
where the functions $h_{j,\P}$ fulfill 
\be
h_{j,\P}(x,y) \in 
\partial \Lpin{\tau_j}(y,f_{\Lpin{\tau_j}, \P, \lb_{j,n}} (x))
  \qquad \forall\,(x,y)\in\cXY\, ,
\ee
see e.g. 
\cite[Thm.\,5.8, Cor.\,5.11]{SteinwartChristmann2008a}.
In the same manner we obtain by straightforward calculations
bounds for the maximal bias of SVMs and Bouligand
influence functions for linear combinations of SVMs, see
\cite{ChristmannVanMessem2008}.
Note that SVMs exist even for heavy-tailed distributions which violate the classical assumption $\Ex_\P|Y|\lt \infty$, which can be shown by
using a trick already used by \cite{Huber1967} where instead of the original loss function
a shifted loss function in the sense 
$L^*_{\tau-\mathrm{pin}}(y,t)
:=
L_{\tau-\mathrm{pin}}(y,t)-L_{\tau-\mathrm{pin}}(y,0)$,
$y,t\in\R$, is used, see \cite{ChristmannVanMessemSteinwart2009}.
This only changes the objective function to be minimized, but not the SVM itself.

\begin{Example}\label{IQR-SVM}[Estimation of scale functions.]
Let $m=2$, $c=(-1,+1)$, $\tau_1\in(0,\frac{1}{2})$ and $\tau_2\in(\frac{1}{2},1)$, e.g. $(\tau_1,\tau_2)=(\frac{1}{4},\frac{3}{4})$ or                                                                                                              $(\tau_1,\tau_2)=(0.05, 0.95)$. 
Then we obtain immediately existence, uniqueness and consistency results
for the difference of the two SVMs based on 
pinball loss functions $\Lpin{\tau_2}$ and $\Lpin{\tau_1}$, respectively. 
In other words, if we denote a $\tau_j$-quantile of the conditional distribution
of $Y$ given $X=x$ by $f_{\tau_j,\P}^*$, 
then the following difference of two SVMs
$$
f_{\Lpin{\tau_2}, \D, \lb_{2,n}}
-
f_{\Lpin{\tau_1}, \D, \lb_{1,n}}
$$
yields an estimator for the difference of 
$f_{\tau_2,\P}^* - f_{\tau_1,\P}^*$.
$\hfill \blacktriangleleft$
\end{Example}

\begin{Example}\label{ASYM-SVM}[Estimation of asymmetry functions.]
Let $m=3$, $c=(+1,-2,+1)$, $\tau_1\in(0,\frac{1}{2})$, $\tau_2=\frac{1}{2}$,
and $\tau_3\in(\frac{1}{2},1)$, e.g. 
$(\tau_1,\tau_2,\tau_3)=(\frac{1}{4},\frac{1}{2}, \frac{3}{4})$ or                                                                                                              $(\tau_1,\tau_2,\tau_3)=(0.05, 0.5, 0.95)$. 
Then we obtain immediately existence, uniqueness and consistency results
for 
$$
f_{\Lpin{\tau_3}, \D, \lb_{3,n}}
-
2 f_{\Lpin{\tau_2}, \D, \lb_{2,n}}
+ 
f_{\Lpin{\tau_1}, \D, \lb_{1,n}},
$$
which gives us an estimator for the difference of 
\be \label{assym} 
( f_{\tau_3,\P}^* - f_{\tau_2,\P}^*) - (f_{\tau_2,\P}^* - f_{\tau_1,\P}^*).
\ee
Let us now choose $\tau\in(0,\frac{1}{2})$ and $\tau_1=1-\tau_3=\tau$, e.g. $\tau=0.05$.
Then the function in {(\ref{assym})} is zero, if, for all $x\in\cX$,
the upper conditional quantile $f_{1-\tau,\P}^*(x)$ differs from the conditional median $f_{0.5,\P}^*(x)$ by the same amount than the conditional median
$f_{0.5,\P}^*(x)$ differs from the lower conditional quantile $f_{\tau,\P}^*(x)$.
Hence the function in {(\ref{assym})} or its supremum norm
can be used as a quantity to measure the amount of asymmetry of the conditional
distribution of $Y$ given $X=x$.
$\hfill \blacktriangleleft$
\end{Example}

It is well-known that the so-called \emph{crossing problem} can occur in quantile regression and that this problem is \emph{not} specific to SVMs, see 
\cite[p.\,55-59]{Koenker2005}. The crossing problem occurs if, for two quantile levels $\tau_1 < \tau_2$, the \emph{estimated} quantile functions $\hat{q}_{\tau_1}, \hat{q}_{\tau_2}$ are in the wrong order for at least one $x\in\cX$, i.e. 
$\hat{q}_{\tau_1}(x) > \hat{q}_{\tau_2}(x)$.
The danger that the crossing problem occurs for a fixed data set is typically small
if $\tau_1$ is close to $0$ and if $\tau_2$ is close to $1$. A numerical method to prevent the crossing problem in kernel based quantile regression
was proposed by \cite{TakeuchiLeSearsSmola2006}.   
  \subsection{Comparison of MAD-type and IQR-type estimation}\label{sec:results3}
    From our point of view, it will often depend on the
application whether an MAD- or an IQR-type SVM is more 
appropriate.

We see three advantages of MAD-type estimation.
(i) One can estimate the heteroscedasticity of $\P(\cdot|x)$ by estimating the conditional median of the absolute residuals $|Y-\hat{f}(x)|$ without estimating \emph{two} conditional quantile functions.
Because in most applications the conditional median (or the conditional mean)
are estimated anyway, one only needs to compute \emph{one} additional SVM instead of
two additional SVMs by the IQR-type approach.
(ii) It can happen that the upper and the lower quantile functions are hard to approximate, e.g., they are not in the RKHSs $H_1$ and $H_2$ which can easily happen
even with the classical Gaussian RBF kernel whose RKHS contains only continuous functions, see 
\cite[Lem.\,4.28, Cor.\,4.36]{SteinwartChristmann2008a} 
whereas the true quantile functions may have jumps.
(iii) It can happen that the \emph{difference} of two quantile functions is easy to estimate, e.g. it is constant, linear, or a polynomial of low order, although the quantile functions $f_{\tau_1,\P}^*$ and $f_{\tau_2,\P}^*$ are complicated.

On the other hand, we see three advantages of IQR-type estimation: 
(i) Greater flexibility by the choice of $(\tau_1,\tau_2)$ whereas the MAD-type approach is based on estimating \emph{one} conditional quantile (which is here  $\tau=\frac{1}{2}$) of the distribution of the \emph{absolute residuals}.
(ii) Greater flexibility by choosing different types of kernels or kernels with different kernel parameters for estimating the upper and the lower quantile functions.
(iii) The IQR-type approach allows the direct estimation of asymmetry or other quantities of interest for the distribution of $Y$ given $X=x$. 

\section{Proofs} \label{sec:appendix}

\subsection{$L_1$ consistency of quantile function estimation by SVMs}

The following lemma strengthens 
\cite[Cor.\ 3.62]{SteinwartChristmann2008a}
in case of the pinball loss function as
convergence in probability is replaced by the stronger
$L_1$-convergence.
\begin{Lemma}\label{lemma-L1-convergence-quantile}
  Let $L$ be the pinball loss with 
  $\tau\in(0,1)$ and let $\P\in\PM(\cX\times\cY)$ be the distribution of
  $(X,Y)$. Assume that $\Ex_{\P} |Y|<\infty$
  and that the conditional quantile function
  $f_{\tau,\P}^\ast:\cX\rightarrow\R$ is $\P_{\cX}$-a.s.\ unique. Then,
  for every $f_n\in L_1(\P_{\cX})$, $n\in\N$, we have 
  $$\lim_{n\rightarrow\infty}\cR{L,\P}{f_n}
    \,=\,\inf_{f\in\mathcal{L}_0(\cX)}\cR{L,\P}{f}
    \qquad\Rightarrow\qquad
    \lim_{n\rightarrow\infty}
    \big\|f_n-f_{\tau,\P}^\ast\big\|_{L_1(\P_{\cX})}\,=\,0\;.
  $$  
\end{Lemma}
\begin{proof}
  Define
  $h_n:\cX\times\cY\rightarrow\R$ by 
  $h_n(x,y)=L\big(y,f_n(x)\big)$ and    
  $h_0:\cX\times\cY\rightarrow\R$ by 
  $h_0(x,y)=L\big(y,f_{\tau,\P}^\ast(x)\big)$.
  Define $c:=\min\{1-\tau,\tau\}$ and note that
  $L(y,t)\geq c|y-t|$ for every $(y,t)\in\cY\times\R$.
  According to \cite[Cor.\ 3.62]{SteinwartChristmann2008a},
  it is already known that $f_n\rightarrow f_{\tau,\P}^\ast$
  in probability (w.r.t.\ $\P_{\cX}$). Therefore, it follows 
  from the continuity of $L$ that
  $h_n\rightarrow h_{0}$
  in probability (w.r.t.\ $\P$). Since
  $$\lim_{n\rightarrow\infty}{\textstyle \int}|h_n|\,d\P\,=\,
    \lim_{n\rightarrow\infty}\cR{L,\P}{f_n}\,=\,
    \inf_{f\in\mathcal{L}_0(\cX)}\cR{L,\P}{f}
    \,=\,\cR{L,\P}{f_{\tau,\P}^\ast}
    \,=\,{\textstyle \int}|h_0|\,d\P\,,
  $$
  the sequence $(h_n)_{n\in\N}$, is uniformly integrable;
  see e.g.\ \cite[Thm.\ 21.7]{bauer2001}.
  Since
  $$\big|f_n(x)\big|\leq\big|y-f_n(x)\big|+|y|\leq
    c^{-1}L\big(y,f_n(x)\big)+|y|=
    c^{-1}h_n(x,y)+|y|
    \quad\forall\,(x,y,n)\in\cX\times\cY\times\N,
  $$
  it follows that
  the sequence $f_n$, $n\in\N$, is uniformly integrable, too.
  Hence convergence in probability of $f_n$, $n\in\N$, implies
  $L_1$-convergence; see e.g.\ \cite[Thm.\ 21.7]{bauer2001}.
\end{proof}
The following theorem strengthens 
\cite[Thm.\ 9.7(i)]{SteinwartChristmann2008a} as
convergence in probability is replaced by the stronger
$L_1$-convergence. The proof coincides with that
of \cite[Thm.\ 9.7(i)]{SteinwartChristmann2008a} apart from
applying Lemma \ref{lemma-L1-convergence-quantile} instead of
\cite[Cor.\ 3.62]{SteinwartChristmann2008a} and therefore is
omitted.
\begin{Theorem}\label{rh-L1consistency}
  Let $\cX$ be a complete measurable space, $\cY\subset\R$
  be closed, $L$ be the pinball loss with $\tau\in(0,1)$, 
  $H$ be a separable RKHS of a bounded kernel $k$ on $\cX$
  such that $H$ is dense in $L_1(\mu)$ for all
  $\mu\in\PM(\cX)$, and $\lambda_n\in(0,\infty)$, $n\in\N$,
  such that $\lim_{n\rightarrow\infty}\lambda_n=0$
  and $\lim_{n\rightarrow\infty}\lambda_n^2n=\infty$.
  Let $\P\in\PM(\cX\times\cY)$ be the distribution of
  $(X,Y)$ and assume that $\Ex_{\P} |Y|<\infty$
  and that the conditional quantile function
  $f_{\tau,\P}^\ast:\cX\rightarrow\R$ is $\P_{\cX}$-a.s.\ unique. 
  Then,
  $$\big\|f_{L,\D,\lambda_n}-f_{\tau,\P}^\ast\big\|_{L_1(\P_{\cX})}
    \;\;\rightarrow\;\;0
    \qquad\text{in probability},
    \qquad n\rightarrow\infty.
  $$
\end{Theorem}

\subsection{Proof of Theorem \ref{MAD-risk-consistency}}

In order to increase the readability of the proof, a
comprehensive notation is needed. Therefore, we
define $L_0:=L_{0.5-\text{pin}}$ and, for 
probability measures $\P_1,\P_2\in\PM(\cX\times\cY)$, we define
\begin{eqnarray*}
  \fPOn&:=&f_{L_0,\P_1,\lambda_{1,n}}\;=\;
  \text{arg}\inf_{f\in H_1} 
  \bigg(\int L_0\big(y,f(x)\big)\,\P_1\big(d(x,y)\big)
        \,+\,\lambda_{1,n}\|f\|_{H_1}^2
  \bigg)\;,\\
  \gPOPTn&:=&
  \text{arg}\inf_{g\in H_2} 
  \bigg(\int L_\varepsilon
          \big(|y-\fPOn(x)|,g(x)\big)\,\P_2\big(d(x,y)\big)
        \,+\,\lambda_{2,n}\|g\|_{H_2}^2
  \bigg)\;.
\end{eqnarray*}
In this definition,
$\P_1$ and $\P_2$ can also be equal to the empirical
measure $\D=\frac{1}{n}\sum_{i=1}^n\delta_{(X_i,Y_i)}$, which 
corresponds to the random sample 
$D=\big((X_1,Y_1),\dots,(X_n,Y_n)\big)$.
That is, the estimate $g_{\D,n}$
defined in (\ref{def-MAD-SVM-non-negative}) 
and (\ref{def-MAD-SVM-pre}), 
is given by
$g_{\D,n}=\max\{\gDDn,0\}$ in this notation. 
We obtain 
$$L_\varepsilon^\prime(y,t)
  \;:=\;\tfrac{\partial}{\partial t}L_\varepsilon(y,t)
  \;=\;\tfrac{1}{2}-\Lambda\big(\tfrac{y-t}{\varepsilon}\big)
  \qquad\forall\,y,t\in\R\;.
$$
Since 
$\big|\frac{\partial}{\partial y}L_\varepsilon(y,t)\big|\leq0.5$
and $\big|\frac{\partial}{\partial t}L_\varepsilon(y,t)\big|\leq0.5$
for every $y,t\in\R$, the following Lipschitz property is fulfilled
\begin{eqnarray}\label{smoothed-pinball-lipschitz}
  \big|L(y_1,t_1)-L(y_2,t_2)\big|\;\leq\;
  0.5|y_1-y_2|+0.5|t_1-t_2|
  \qquad\forall\,y_1,y_2,t_1,t_2\,\in\,\R\;.
\end{eqnarray}
An easy calculation shows that
(\ref{smoothed-pinball-lipschitz}) implies
\begin{eqnarray}\label{smoothed-pinball-lipschitz-risks}
  \big|\LErisk(f_1,g_1)-\LErisk(f_2,g_2)\big|
  \;\leq\;0.5\big\|f_1-f_2\|_{L_1(\P_{\cX})}+
          0.5\big\|g_1-g_2\|_{L_1(\P_{\cX})}
\end{eqnarray} 
for all $f_1,f_2,g_1,g_2\in \mathcal{L}_1(\P_{\cX})$.
Note that, by construction,  
$0\leq L_0-L_\varepsilon\leq\varepsilon$,
which implies
\begin{eqnarray}\label{smoothed-pinball-risk}
  \LErisk(f,g)\;\leq\;\LZrisk(f,g)\;\leq\;
  \LErisk(f,g)+\varepsilon
  \qquad\forall\,f,g\in \mathcal{L}_1(\P_{\cX})\,.
\end{eqnarray}
It is obvious from the definition (\ref{def-risk-for-mad})
of the risk $\LZrisk(f,g)$
that replacing negative values of the function $g$ by 0
reduces the risk. Hence,
the definitions imply 
$\LZrisk(f_{0.5,\P}^\ast,g_{\D,n})\leq\LZrisk(f_{0.5,\P}^\ast,\gDDn)$
and it follows from (\ref{smoothed-pinball-risk}) that
\begin{eqnarray}\label{MAD-main-theorem-p4001}
  \lefteqn{
  \LZrisk(f_{0.5,\P}^\ast,g_{\D,n})
           -\!
            \inf_{g\in \mathcal{L}_0(\cX)}\!
              \LZrisk(f_{0.5,\P}^\ast,g)\,\leq\,
  }\nonumber\\
  & &\quad\leq\;
         \LErisk(f_{0.5,\P}^\ast,\gDDn)
           -\!\inf_{g\in \mathcal{L}_0(\cX)}\!
              \LErisk(f_{0.5,\P}^\ast,g)+\varepsilon\;.\qquad
\end{eqnarray}
Applying the triangular inequality yields
\begin{eqnarray}\label{MAD-main-theorem-p4002}
  \LErisk(f_{0.5,\P}^\ast,\gDDn)\!\leq\!
         \Big|\LErisk(f_{0.5,\P}^\ast,\gDDn)\!-\!\LErisk(\fPn,\gPPn)\Big|
         \!+\!\LErisk(\fPn,\gPPn).
\end{eqnarray}
Next, define
\begin{eqnarray*}
  \Delta_1^{(n)}\!\!\!\!\!
  &:=&\!\!\!\!
     \big\|\gDDn-\gPDn\|_{L_1(\P_{\cX})}\,,\;\;\;
    \Delta_2^{(n)}\,:=\,\big\|\gPDn-\gPPn\|_{L_1(\P_{\cX})}\;, \\
  \Delta_3^{(n)}\!\!\!\!\!
  &:=&\!\!\!\!
    \big\|f_{0.5,\P}^\ast-\fPn\|_{L_1(\P_{\cX})}\,,\;\;\;
  \Delta_4^{(n)}\,:=\,
    \Big(\LErisk(\fPn,\gPPn)-
          \inf_{g\in \mathcal{L}_0(\cX)}\LErisk(f_{0.5,\P}^\ast,g)
    \Big).
\end{eqnarray*}
Then, it follows from (\ref{MAD-main-theorem-p4001}),
(\ref{MAD-main-theorem-p4002}),  
(\ref{smoothed-pinball-lipschitz-risks}), and another application of
the triangular inequality that
\begin{eqnarray}\label{MAD-main-theorem-p4003}
  0\leq\LZrisk(f_{0.5,\P}^\ast,g_{\D,n})
           -\!\!
            \inf_{g\in \mathcal{L}_0(\cX)}\!\!
              \LZrisk(f_{0.5,\P}^\ast,g)
         \leq
          0.5\Big(\!\Delta_1^{(n)}\!+\Delta_2^{(n)}\!+\Delta_3^{(n)}\!
             \Big)
          \!+\Delta_4^{(n)}
          \!+\varepsilon \,.
\end{eqnarray}
Each of the four summands $\Delta_1^{(n)},\dots,\Delta_4^{(n)}$ 
will be considered separately in the 
following four parts. In order to
prove the theorem, it is enough to show that $\Delta_1^{(n)}$
and $\Delta_2^{(n)}$ converge to 0 in probability (Part 1 and Part 2), 
that $\Delta_3^{(n)}$ 
converges to 0 (Part 3), and that the limit superior of 
$\Delta_4^{(n)}$ is not larger than 0 (Part 4);
the terms $\Delta_3^{(n)}$ and $\Delta_4^{(n)}$ are non-stochastic.
Note that (\ref{smoothed-pinball-lipschitz-risks})
and Theorem \ref{rh-L1consistency} imply, for $n\rightarrow\infty$,
the convergence in probability of
$$\big|\LZrisk\big(f_{0.5,\P}^\ast,g_{\D,n}\big)-
       \LZrisk\big(f_{L_0,\D,\lambda_{1,n}},g_{\D,n}\big)
  \big|
  \;\leq\;
  0.5\big\|f_{0.5,\P}^\ast-f_{L_0,\D,\lambda_{1,n}}
     \big\|_{L_1(\P_{\cX})}
  \;\rightarrow\;0\;.
$$
\textbf{Part 1:\,}
For $D=\big((X_1,Y_1),\dots,(X_n,Y_n)\big)$, define
$$\Q_D\;=\;\frac{1}{n}\sum_{i=1}^n\delta_{(X_i,|Y_i-\fPn(X_i)|)}
  \qquad\text{and}\qquad
  \tilde{\Q}_D
  \;=\;\frac{1}{n}\sum_{i=1}^n\delta_{(X_i,|Y_i-\fDnn(X_i)|)}\;.
$$
For every $(x,y)\in\cX\times\cY$, define
$h_{\D,n}(x,y)=L^\prime_{\varepsilon}(y,\gPDn(x))$.
Then, it follows from the
representer theorem \cite[Cor.\ 5.10]{SteinwartChristmann2008a} 
that
\begin{eqnarray}\label{MAD-main-theorem-p4004}
  \lefteqn{
    \big\|\gDDn-\gPDn\big\|_{H_2}\,\leq\,
    \lambda_{2,n}^{-1}
       \big\|\Ex_{\tilde{\Q}_D}h_{\D,n}\Phi_2-\Ex_{\Q_D}h_{\D,n}\Phi_2
       \big\|_{H_2}\,\leq
  }\nonumber\\
  &\leq&
       \frac{1}{\lambda_{2,n}\,n}
       \sum_{i=1}^n
          \Big|h_{\D,n}(X_i,|Y_i-\fDnn(X_i)|)-h_{\D,n}
                 (X_i,|Y_i-\fPn(X_i)|)
          \Big|
          \!\cdot\!\big\|\Phi_2(X_i)\big\|_{H_2}\,.\qquad         
\end{eqnarray}
According to the boundedness of $k_1$ and $k_2$,
we will use the well-known inequalities
\begin{eqnarray}\label{MAD-main-theorem-p4005}
  \|\Phi_j(x)\|_{H_j}\,\leq\,\|k_j\|_\infty
  \quad\forall\,x\in\cX
  \qquad\text{and}\qquad
  \|f\|_\infty\,\leq\,\|k_j\|_\infty\|f\|_{H_j}
  \quad\forall\,f\in H_j 
\end{eqnarray}
for every $j\in\{1,2\}$;
see \cite[p.\ 124]{SteinwartChristmann2008a}.
Then, the definition of $h_{\D,n}$ and the easy to prove
Lipschitz property 
$\,\big|L^\prime_{\varepsilon}(y_1,t)-L^\prime_{\varepsilon}(y_2,t)
   \big|
 \,\leq\,\varepsilon^{-1}|y_1-y_2|
$ for all $y_1,y_2,t\in\R$
imply
\begin{eqnarray}
  \lefteqn{
  \big\|\gDDn-\gPDn\big\|_{H_2}
  \stackrel{(\ref{MAD-main-theorem-p4004},\ref{MAD-main-theorem-p4005})%
           }{\leq}
  \frac{\|k_2\|_\infty}{\lambda_{2,n}\,n}\!
       \sum_{i=1}^n\!
          \Big|h_{\D,n}(X_i,|Y_i\!-\!\fDnn(X_i)|)
               -h_{\D,n}(X_i,|Y_i\!-\!\fPn(X_i)|)
          \Big|          
  }\nonumber\\
  &\!\!\leq&\!\!\!\! \label{lipschitz-continuous-derivative-of-loss-1}
       \|k_2\|_\infty\lambda_{2,n}^{-1}
        \sup_{t}\sup_{x,\,y}
          \Big|L^\prime_{\varepsilon}\big(|y\!-\!\fDnn(x)|,t\big)
               -L^\prime_{\varepsilon}\big(|y\!-\!\fPn(x)|,t\big)
          \Big| \\      
  &\!\!\leq&\!\!\!\!  \label{lipschitz-continuous-derivative-of-loss-2}
        \|k_2\|_\infty\varepsilon^{-1}\lambda_{2,n}^{-1}
        \sup_{x,\,y}
          \Big|\big|y\!-\!\fDnn(x)\big|
               -\big|y\!-\!\fPn(x)\big|
          \Big|
      \;\leq\;\|k_2\|_\infty\varepsilon^{-1}\lambda_{2,n}^{-1}
          \big\|\fDnn\!-\!\fPn\big\|_{\infty}\qquad\;\\
  &\!\!\stackrel{(\ref{MAD-main-theorem-p4005})}{\leq}&\!\!\!\! 
         \|k_1\|_\infty\|k_2\|_\infty\varepsilon^{-1}\lambda_{2,n}^{-1}
          \big\|\fDnn-\fPn\big\|_{H_1}\;. \nonumber
\end{eqnarray}
Next, it follows from the representer theorem 
\cite[Cor.\ 5.10]{SteinwartChristmann2008a} that
there is an $h_{\P,n}\in\cL_\infty(\cX)$ such that 
$\|h_{\P,n}\|_\infty\leq0.5$
and
\begin{eqnarray}\label{proof-part-one-1001}
  \big\|\fDnn-\fPn\big\|_{H_1}\;\leq\;
  \lambda_{1,n}^{-1}
           \bigg\|\frac{1}{n}\!\sum_{i=1}^n
                    \big(h_{\P,n}(X_i,Y_i)\Phi_1(X_i)
                         -\Ex_{\P}h_{\P,n}\Phi_1
                    \big)
           \bigg\|_{H_1}\;.
\end{eqnarray}
Define 
$B:=\sup_{x,y}\|h_{\P,n}(x,y)\Phi_1(x)\|_{H_1}\leq 0.5\|k_1\|_\infty$ 
and
fix any $\eta>0$. Then it follows from (\ref{proof-part-one-1001})
and
Hoeffding's inequality 
\cite[Chapter 3]{Yurinsky1995}
that, for $n\rightarrow\infty$,
$$\P^n\Big(\lambda_{2,n}^{-1}\big\|\fDnn\!-\!\fPn\big\|_{H_1}\geq\eta\Big)
  \,\leq\,
  \exp\bigg(-\frac{3}{8}\cdot
                    \frac{\eta^2\lambda_{1,n}^2\lambda_{2,n}^2n}{%
                          \eta\lambda_{1,n}\lambda_{2,n}B+3B^2%
                         }
             \bigg)
         \;\;\longrightarrow\;\;0\;,
$$
because 
$\lim_{n\rightarrow\infty}\lambda_{1,n}^2\lambda_{2,n}^2n=0$.
That is, we have shown that 
$\Delta_1^{(n)}=\big\|\gDDn-\gPDn\big\|_{H_2}$ converges
to 0 in probability w.r.t.\ $\P^n$. 

\textbf{Part 2:\,}
Define $L_{\varepsilon;n}(x,y,t)=L_\varepsilon(|y-\fPn(x)|,t)$.
This yields
$$\gPPOn\;:=
  \text{arg}\inf_{g\in H_2} 
  \left(\int L_{\varepsilon;n}\big(x,y,g(x)\big)\,\P_1\big(d(x,y)\big)
        \,+\,\lambda_{2,n}\|g\|_{H_2}^2
  \right)
  \quad\;\forall\,\P_1\in\PM(\cX\times\cY)\;.
$$
Hence, 
for $h_{\P,n}(x,y)=L^\prime_\varepsilon(|y-\fPn(x)|,t)$,
the representer theorem 
\cite[Cor.\ 5.10]{SteinwartChristmann2008a} implies that
$$\big\|\gPDn-\gPPn\big\|_{H_2}\,\leq\,
  \lambda_{2,n}^{-1}
       \bigg\|\frac{1}{n}\sum_{i=1}^n h_{\P,n}\Phi_2
              -\Ex_{\P}h_{\P,n}\Phi_2
       \bigg\|_{H_2}\,.
$$
For 
$B:=\sup_{x,y}\|h_{\P,n}(x,y)\Phi_2(x)\|_{H_1}
 \leq 0.5\|k_2\|_\infty
$,
it follows from Hoeffding's inequality \cite[Chap.\ 3]{Yurinsky1995} 
and $\lambda_{2,n}^{2}n\rightarrow\infty$
that $\Delta_2^{(n)}=\big\|\gPDn-\gPPn\big\|_{H_2}$ 
converges to 0 in probability.

\textbf{Part 3:\,}
Since 
$\,\lim_{n\rightarrow\infty}\cR{L_0,\P}{\fPn}
 =\inf_{f\in\mathcal{L}_0(\cX)}\cR{L_0,\P}{f}
$
as shown in 
\cite[p.\ 338]{SteinwartChristmann2008a},it follows from 
Lemma \ref{lemma-L1-convergence-quantile}
that
\begin{eqnarray}\label{L1-convergence-of-SVM-to-cond-median} 
  \lim_{n\rightarrow\infty}\Delta_3^{(n)}\;=\;
  \lim_{n\rightarrow\infty}\|f_{\P}^\ast-\fPn\|_{L_1(\P_{\cX})}
  \;=\;0\;.
\end{eqnarray}
\textbf{Part 4:\,}
For every $g\in H_2$, define the approximation error function
(where we use the notation (\ref{def-risk-for-mad}))
$$A_g\,:\;L_1(\P_{\cX})\times\R\;\rightarrow\;\R\,,\quad\;
  (f,\lambda)\;\mapsto\;
  \LErisk(f,g)+\lambda\|g\|_{H_2}^2\,-\inf_{g_0\in\cL_0(\cX)}
    \LErisk\big(f_{0.5,\P}^\ast,g_0\big)\,.
$$
Note that the assumption $\Ex_{\P} |Y|<\infty$ implies that
$|A_g(f,\lambda)| <\infty$ such that $A_g$ is well defined.
It follows from the Lipschitz property 
(\ref{smoothed-pinball-lipschitz}) of $L_\varepsilon$
that 
$A_g$ is continuous for every $g\in H_2$ and, therefore, the map
$(f,\lambda)\mapsto\inf_{g\in H_2}A_g(f,\lambda)$ is upper 
semicontinuous. Hence, (\ref{L1-convergence-of-SVM-to-cond-median})
implies
$$\limsup_{n\rightarrow\infty}\Delta_4^{(n)}\;\leq\;
  \limsup_{n\rightarrow\infty}\inf_{g\in H_2}A_g(\fPn,\lambda_{2,n})
  \;\leq\;\inf_{g\in H_2}A_g(f_{0.5,\P}^\ast,0)\;=\;0\,,
$$
where the last equality follows, because
the assumption that $H_2$ is dense in $L_1(\P_{\cX})$
guarantees 
$\inf_{g_0\in\cL_0(\cX)}
    \LErisk\big(f_{0.5,\P}^\ast,g_0\big)
 =\inf_{g\in H_2}
    \LErisk\big(f_{0.5,\P}^\ast,g_0\big)
$
according to \cite[Theorem 5.31]{SteinwartChristmann2008a}.
\hfill$\blacksquare$

{\small \bibliography{christmann,hable}}
\end{document}